\documentclass[letterpaper,10 pt,conference]{ieeeconf}

\usepackage[utf8]{inputenc}
\usepackage{graphicx}
\usepackage{subcaption}
\usepackage{amssymb}

\usepackage{amsmath}

\usepackage{amsfonts}
\usepackage[english]{babel}
\usepackage[bottom]{footmisc}
\usepackage{wrapfig}
\usepackage[ruled]{algorithm}
\usepackage{algpseudocode}
\usepackage{siunitx}
\usepackage{array}
\usepackage{placeins}
\usepackage{hyperref}
\usepackage{url}
\usepackage{multirow}
\usepackage{adjustbox}

\title{Safety Guarantees for Iterative Predictions with Gaussian Processes}

\author{Kyriakos Polymenakos, Luca Laurenti, Andrea Patane, Jan-Peter Calliess,\\
Luca Cardelli, Marta Kwiatkowska, Alessandro Abate \& Stephen Roberts}

\newtheorem{theorem}{Theorem}
\newtheorem{proposition}{Proposition}
\newtheorem{corollary}{Corollary}
\newtheorem{lemma}{Lemma}

\newcommand{\bi}[1]{\begin{itemize}}
\newcommand{\ei}[1]{\end{itemize}}

\date{March 2020}

\begin{document}

\maketitle
\begin{abstract}

Gaussian Processes (GPs) are widely employed in control and learning because of their principled treatment of uncertainty. 
However, tracking uncertainty for iterative, multi-step predictions in general leads to an analytically intractable problem.
While approximation methods exist, they do not come with guarantees, making it difficult to estimate their reliability and to trust their predictions. 
In this work, we derive formal probability error bounds for iterative predictions with GPs. 
Building on GP properties, we bound the probability that random trajectories lie in specific regions around the predicted values. 
Namely, given a tolerance $\epsilon > 0 $, we compute regions around the predicted trajectory values, such that GP trajectories are guaranteed to lie inside them with probability at least $1-\epsilon$.
We verify experimentally that our method tracks the predictive uncertainty correctly, even when current approximation techniques fail.
Furthermore, we show how the proposed bounds 
can incorporate a given control law, and effectively bound the trajectories of the closed-loop system.
\footnote{A shorter version of this paper was presented in the 59th Conference on Decision and Control (CDC 2020).}

\end{abstract}

\section{Introduction}

Gaussian processes (GPs) have been extensively used for modelling due to the variety of suitable properties they possess:
they are probabilistic models, providing uncertainty estimates on their predictions;
they are non-parametric, effectively adjusting the model complexity to the data, 
and finally they are usually data-efficient \cite{gpbook}.
In plenty of scenarios 
(e.g.\ planning, forecasting, and time-series modelling)
one needs to make several, possibly correlated, 
predictions at once (the second prediction is made before the first one can be evaluated versus a ground truth, and so on).
For this we can discern two options: either train multiple models, each one predicting at different time-scales, 
or use a single model, that iteratively computes predictions that get in turn fed back as input to the model in the next step. 
We refer to the latter as \emph{iterative predictions} and \emph{iterative planning}.
%
%

Of particular interest for the iterative planning scenario is the model-based reinforcement learning setting,
where a GP model is used to evaluate a candidate control policy on the  system. 
The evaluation requires the model to provide predictions for the system's state
over multiple time-steps under the proposed policy. 
It is important in these cases to have a realistic assessment of the 
error on the predictions,
as this allows quantification of the risk of costly system failures, like
collisions with obstacles or financial losses, and analysis of safety-critical applications.
In such settings, we require predictions that are not only accurate on average,
but also provide robust, (probabilistically) guaranteed worst-case accuracy.

Unfortunately, as GP models output probability distributions, iterative planning poses the problem of prediction over successive \textit{noisy inputs} (i.e.\ with a distribution placed over the input space).
This leads to an analytically intractable problem for such non-linear input-output mappings.
While several approximation techniques have been proposed \cite{girard2003gaussian,Vinogradska2018}, to the best of our knowledge, none of them provides guarantees, in the form of formal error bounds on their estimations, making it difficult to estimate reliability and trust predictions in application scenarios.

In this work we provide a probabilistic bound for iterative predictions with GPs and develop a method for its explicit computation.
Given a user-defined tolerance $\epsilon > 0$, our method works by computing probabilistic bounds at each prediction step and propagating them over multiple time-steps in the form of intervals. The GP trajectories are guaranteed to lie inside these intervals at each time step with probability at least $1 - \epsilon$.  
In practice, this allows us to perform long-term predictions for the GP trajectory 
with the prediction provably staying within known bounds with a specified probability.
We further show how the bound can be used within a reinforcement learning scenario, 
in order to guarantee the safety of proposed control policies.
We provide an algorithmic framework for the explicit computation of every value involved in the bound calculation, directly and efficiently from data, so that the bound can be computed independently of the form of the learned GP.



On a set of case studies, we show how our method can correctly provide probabilistic bounds that account for the GP uncertainty over its trajectories. 
Finally, we illustrate how our bound can be successfully employed to verify both open loop and feedback policies and therefore guarantee the safety of proposed controllers for the learned GP. In summary, the paper makes the following main contributions:
\begin{itemize}
    \item We develop a formal bound, for iterative prediction settings, on the probability that the trajectories of a GP lie inside a specific region. We provide explicit computational techniques for calculating the bound.
    \item We incorporate control laws and take into account their effects in the model's trajectories.
    \item We provide experimental validations of our method, highlighting cases in which a competitive state-of-the-art method fails to properly propagate the GP uncertainty. We provide case studies on certification of open-loop and feedback policies.
\end{itemize}

\section{Related Work}
Performing iterative predictions, and using them for planning, is an extensively studied problem across various model types \cite{ale_survey,green2012linear,kahn17,vuong2019uncertainty}.

\sloppy
In particular, GP multi-step-ahead prediction is generally achieved using heuristic approximations \cite{girard2003gaussian}.
The most widely used approach is Moment Matching (MM) which computes a Gaussian approximation over the (non-Gaussian) output distribution of a GP for a noisy input \cite{girard2003gaussian,candela2003propagation}.
The uncertainty estimated in this fashion can then be leveraged to learn control policies in frameworks such as PILCO (Probabilistic Inference for Learning COntrol) \cite{pilco,pil_thesis}. 
During the last few years, various extensions of PILCO have been proposed \cite{pilco2,deisenroth2014multi,kupcsik2013data,mcallister2017data}. 
For example, in \cite{pilconn}  GPs have been replaced with neural networks, while in 
\cite{polymenakos2019safe} an emphasis on safety is given. 
However, building on Gaussian approximations, all the cited  approaches inherently 
fail to take into account multi-modal behaviour and tend to underestimate uncertainty.
As such, the synthesised policies are not guaranteed to be safe.
Our method on the other hand comes with probabilistic guarantees that allow us to compute the subregions of the input space in which the trajectories of the analysed GP are bound to lie with high probability.
As such it provides formal, guaranteed bounds on the GP trajectories and makes no particular assumptions on the GP model, enabling its use in safe reinforcement learning scenarios \cite{safe_survey}.

Numerical approximations exist for multi-step-ahead predictions \cite{Vinogradska2018} where the output distribution is directly approximated by using quadrature formulas and, in principle, worst-case scenario error bounds could be computed using existing techniques for numerical quadrature \cite{vinogradska2016stability}. 
However, the analysis that leads to the bounds proposed in \cite{vinogradska2016stability} is focused on 
stability, with the assumption that trajectories monotonically decrease the distance to a target state, 
and the authors explicitly exclude trajectories that move away from the target state before eventual convergence and stabilisation. In \cite{Vinogradska2018}, where more general tasks are solved, no formal bounds are provided. Our algorithm instead provides valid probabilistic bounds for the general case.

Interestingly, \cite{Koller2018} focus on bounding the modelling error, that is the difference between the underlying system dynamics and the learnt GP model, which is a complementary problem to the one tackled in this work, and employ moment-matching to propagate the uncertainty for multiple time-steps. 
In order to compute error bounds they assume that the underlying function describing the system dynamics, that is approximated by the GP, has a bounded RKHS norm and use existing results for this setting \cite{srinivas2012information}.  
However, their bounds require the computation of constants very difficult to compute in practice. 
In contrast, in this paper we assume that the underlying function is a sample from a GP (and hence we do not consider any possible model mismatch) and derive formal bounds whose required constants are directly computed. 

%
%

Formal and probabilistic guarantees for GPs have been discussed in \cite{cardelli2018robustness} and \cite{blaas2019robustness} for regression and classification with GPs, respectively. 
Albeit formal, these methods cannot be directly applied to multi-step-ahead predictions scenarios as they are designed for GPs over single input points. 
Whereas, our method, by propagating  probabilistic bounds through each time step is applicable to multi-step ahead prediction scenarios and can be used in reinforcement learning settings to verify controller safety.


\section{Bounds for Multi-step Ahead Predictions with Gaussian Processes}
Given  an input space $U \subset \mathbb{R}^m$ and a time horizon $[0,H]$, for $t \in \{0,\ldots,H-1\}\subset \mathbb{N}$ we consider a stochastic dynamical system\footnote{Throughout the paper bold math symbols are used for random variables.} 
\begin{align}
         \label{Eqn:OriginalProcess} \mathbf{x}_{t+1}= \mathbf{f}(\mathbf{x}_t,u_t), \, \, \, u_t \in U,
     \end{align} 
where we assume that for $\mathrm{x}\in X \subset \mathbb{R}^n$, $\mathbf{f}(\mathrm{x},u_t) \sim \mathcal{N}(\mu_{\mathrm{x}}^{\mathbf{f}},\Sigma_{\mathrm{x},\mathrm{x}}^{\mathbf{f}})$ that is, $\mathbf{f}(\mathrm{x},u_t)$ is normally distributed  with mean  vector $\mu_{\mathrm{x}}^{\mathbf{f}}$  
and covariance matrix $\Sigma_{\mathrm{x},\mathrm{x}}^{\mathbf{f}}$\footnote{For simplicity we drop the dependence on $u_t$ in both mean vector and covariance matrix.}. 
Mean and variance of  $\mathbf{f}^i(\mathrm{x},u_t)$, the $i$-th component of $\mathbf{f}(\mathrm{x},u_t)$, are denoted with $\mu_{\mathrm{x}}^{\mathbf{f},i}$ and $\Sigma_{\mathrm{x},\mathrm{x}}^{\mathbf{f},(i,i)}$.
Intuitively, $\mathbf{x}_t$ is a discrete-time stochastic process, whose time evolution depends on an input signal taking values in $U.$
A parametric memory-less and deterministic policy $ \pi^{\theta} : X \rightarrow U $ with parameters $\theta$ is a function that assigns a control input given the current state. 
By iterating Eqn. \eqref{Eqn:OriginalProcess}, we have that, for $t > 0$, $\mathbf{x}_{t}$ is a random variable as it is the output of process $\mathbf{f}$.
As such multi-step ahead predictions poses the problem of predictions over noisy inputs.


\subsection{Prediction over noisy inputs} 
For a given $\mathrm{x} \in X,u \in U$ we have that $\mathbf{f}(\mathrm{x},u)$ is a Gaussian random variable.
However, if $\mathbf{x}_t$ is a random variable itself (which is the case for prediction over noisy inputs), then $\mathbf{f}(\mathbf{x}_t,u)$ is generally not Gaussian anymore and its distribution is in general analytically intractable. 
In particular, we have that
 $$ \mathbf{f}(\mathbf{x}_t,u) \sim \int p(\mathbf{x}_{t+1} | \mathrm{x},u) p(\mathbf{x}_t=\mathrm{x}) d{x},$$
where $p(\mathbf{x}_{t+1} | {x},u)$ is the (normal) distribution of $\mathbf{f}(x,u) $ and $p(\mathbf{x}_t=x)$ is the distribution of $\mathbf{x}_t$.
As a consequence, the predictive distribution for $\mathbf{\mathbf{x}}_{t+1}$ is not Gaussian and approximations are required  \cite{girard2003gaussian}.

In this paper, given $\mathbf{x}_t$, we consider a predictor $\hat {x}_t$ for $\mathbf{x}_t$, such that
\begin{align}
    \label{Eqn:EstimatorGP}
    \hat{\mathrm{x}}_{t}=g(\hat{\mathrm{x}}_{t-1},u_{t-1}), 
\end{align}
where $g(\hat{\mathrm{x}}_{t-1},u_{t-1})$ is a deterministic function. 
That is, $\hat x_t$ is a deterministic process that predicts the value of $\mathbf{x}_t$.
For instance, we could have that $\hat{\mathrm{x}}_t$ equals the mean of $\mathbf{x}_t$, as estimated with moment matching techniques \cite{girard2003gaussian}, but any other deterministic function will work for the results presented in this paper.

In what follows, in Theorem \ref{Prop:BoundErrorRecursive} we compute a probabilistic bound on the error between $\hat x_t$ and $\mathbf{x}_t$.
The bound has a recursive structure, as the uncertainty needs to be propagated over multiple prediction steps.
Please note that is not a modelling error, coming from the GP imperfectly capturing the behaviour of an underlying system, 
but comes solely from propagating the uncertainty while performing iterative predictions.
Then, in Corollary \ref{Corollary:TestTUbeTraectory} we show that, given an $\epsilon>0,$ this bound can be used to build a tube around $\hat x_t$ such that at each time step the trajectories of  $\mathbf{x}_t$ are guaranteed to be within such a tube with probability at least $1-\epsilon.$  
For any safe region $S \subset X$ we can hence produce certificates on whether GP trajectories will lie inside that region with high probability or not. 


\subsection{Bounds for Multi-Step Ahead Predictions}
Consider the random variable on the error 
at time $t$, i.e.\ $\mathbf{e}_t=| \mathbf{x}_t - \hat x_t |_1$ and a constant $K_t>0$. In Theorem \ref{Prop:BoundErrorRecursive} we compute $P( \mathbf{e}_t > K_t  ),$
that is the probability that the error between $\mathbf{x}_t$ and $\hat x_t$ is greater than $K_t$.
\noindent
\begin{theorem}
\label{Prop:BoundErrorRecursive}
For any $K>0$ and ${\mathrm{x}^*}\in X$, let $I_{ x*}^K=\{ \mathrm{x} \in X \, : \, |{x}^* - \mathrm{x}|_1\leq K\}$. 
Assume $\mathbf{x}_0 \sim \mathcal{N}(\mu_{0},\Sigma_{0,0})$.
Then, for arbitrary constants $K_{t+1},K_{t}>0$, it holds that
\begin{align*}
  P(\mathbf{e}_{t+1}> K_{t+1}) \,  \leq&\,
  P(\sup_{x \in I_{\hat x_{t}}^{K_{t}} }|\hat x_{t+1}- \mathbf{f}(\mathrm{x},u_{t}) |_i>{K_{t+1}})\\ & P(\mathbf{e}_{t} \leq K_t)
  +\, P(\mathbf{e}_{t}>K_t), 
\end{align*}  
with
$ P(\mathbf{e}_0 > K_0 )=1- \int_{I^{K_0}_{\hat x_0}} \mathcal{N}(z\,|\,\mu_{0},\Sigma_{0,0}) dz$
for any $K_0>0$, $\mu_{0}$ and $\Sigma_{0,0}$ are the mean and covariance of $\mathbf{x}_0$.
\end{theorem}
The proof of the above theorem is reported in Section \ref{Sec:proofs}.
The resulting bound in Theorem \ref{Prop:BoundErrorRecursive} is recursive. 
Hence, in order to estimate the prediction error at time $t$, we need to compute the prediction error at the previous time steps, which is propagated over time through the bound.
The recursion terminates as the distribution for $\mathbf{x_0}$, that is the initial condition, is given.
Intuitively $K_t$ is a parametric cutoff threshold for the distance at time $t$, and the resulting bound at time $t+1$, that is $\mathbf{e}_{t+1}$, is the sum of the contribution given by assuming that $\mathbf{e}_t \leq K_t$ and by the contribution when assuming $\mathbf{e}_t > K_t$ (and remains valid for any value of $K_t$).

Note that the bound in Theorem \ref{Prop:BoundErrorRecursive} requires the computation of 
$P(\sup_{{\mathrm{x}} \in I_{\hat {\mathrm{x}}_{t}}^{K_t} }$ $|g(\hat {\mathrm{x}}_t,u_t)-\mathbf{f}({\mathrm{x}},u_{t}) |_1>{K_{t+1}})$, that is the probability that the supremum of a stochastic process is greater than a given threshold.
This is in general a difficult problem \cite{adler2009random}. However,  $\mathbf{f}({\mathrm{x}},u_{t})$ is a Gaussian process and $g(\hat {\mathrm{x}}_t,u_t)$ a constant. Therefore, we can use the result from \cite{cardelli2018robustness}, where bounds for the supremum of a GP have been derived.
These are extended to the current setup in the following proposition. 
\begin{proposition}
\label{Prop:SupremumGP}
Let $\mu({\mathrm{x}},\hat {\mathrm{x}}_t)=g(\hat {\mathrm{x}}_t,u_t)-\mu_{\mathrm{x}}^{\mathbf{f}}$.
Assume $I_{\hat {\mathrm{x}}_{t}}^{K_t}$ is a hyper-cube with side length $D$.
 For  $i \in \{1,...,n\}$ let
    \begin{align*}
      \bar \eta^i =  & \frac{K_{t+1}\, -\, \sup_{{\mathrm{x}} \in I_{\hat {\mathrm{x}}_{t}}^{K_t}}|\mu({\mathrm{x}},\hat {\mathrm{x}}_{t})|_1}{n}\,   - \\
        & \quad \quad \quad \quad \quad  12 \int_{0}^{\lambda^i} \sqrt{\ln \left(\big( \frac{\sqrt{N} L^{i}_{{\hat {\mathrm{x}}_t}} D\, }{ z}+ 1\big)^n  \right)}dz,
    \end{align*}
    with $\lambda^i = \frac{1}{2}\sup_{{\mathrm{x}}_1,{\mathrm{x}}_2 \in I_{\hat {\mathrm{x}}_{t}}}^{K_t} d^{(i)}_{{\hat {\mathrm{x}}}}({\mathrm{x}}_1,{\mathrm{x}}_2)$ and $n$ being the dimension of the state space.
   For each $i \in \{1,...,n\}$ assume $\bar \eta^i > 0$. Then, it holds that
    \begin{equation*}
       P(\sup_{\mathrm{x} \in I_{\hat {\mathrm{x}}_{t}}^{K_t} }|g(\hat {\mathrm{x}}_t,u_t)-\mathbf{f}(\mathrm{x},u_t) |_1>{K_{t+1}}) \leq 2 \sum_{i=1}^{n} e^{ -\frac{(\bar \eta^i)^2}{2 \xi^{(i)} }},
    \end{equation*}
    where $
    \xi^{(i)} = \sup_{{\mathrm{x}}\in  I_{\hat {\mathrm{x}}_{t}}^{K_t}} \Sigma^{\mathbf{f},(i,i)}_{{{\mathrm{x}}},{\mathrm{x}}}, $
    \begin{align*}
    &d^{(i)}_{{\hat {\mathrm{x}}_t}}({\mathrm{x}}_1,{\mathrm{x}}_2)=\sqrt{\mathbb{E}[( \mathbf{f}^{i}({\mathrm{x}}_2,u_t)-\mu^{\mathbf{f},i}_{{\mathrm{x}}_2} -(\mathbf{f}^{i}({\mathrm{x}}_1,u_t)-\mu^{\mathbf{f},i}_{{\mathrm{x}}_1}))^2]}
     \end{align*}   
    and $L_{\hat {\mathrm{x}}_t}^{i}$ is a local Lipschitz constant for $d^{(i)}_{\hat {\mathrm{x}}_t}.$
\end{proposition}

  
By using the upper bound of Proposition \ref{Prop:SupremumGP} in Theorem \ref{Prop:BoundErrorRecursive} we can propagate the bound through time for any value of $K_t > 0 $, $t = 0,\ldots,H$.
This give us the degree of freedom necessary to iteratively select, given $K_t$, the values for $K_{t+1}$ that meet an a-priori specified probabilistic error $ \epsilon > 0 $.
To do this it suffices to evaluate the one-step bound resulting from the combination of Proposition \ref{Prop:SupremumGP} and Theorem \ref{Prop:BoundErrorRecursive}, and choose the smallest value of $K_{t+1}$ such that $ P( \mathbf{e}_{t+1} > K_{t +1}  ) < \epsilon$.
  
\begin{corollary}
\label{Corollary:TestTUbeTraectory}{(of Theorem \ref{Prop:BoundErrorRecursive})}
For any $\epsilon>0$ pick the smallest $K_0,...,K_H$ such that for any $t \in \{0,...,H \}$ we have that $ P( \mathbf{e}_t > K_t  )< \epsilon$.
Then, this implies that 
$$\forall t \in \{0,...,H \}, \quad P( \mathbf{x}_t \in I_{\hat x_t}^{K_t})> 1 - \epsilon.$$ 
\end{corollary}
As a result we can compute a sequence of subsets $I_{\hat x_t}^{K_t}$ of the state space such that the GP trajectories are bounded to stay inside them with probability at least $1-\epsilon$ at each time step. 
Given a safe region $S \subseteq X$ we can hence produce a certificate on the $GP$ trajectories lying inside $S$ with probability at least $1-\epsilon$ by checking the intersection between the $I_{\hat x_t}^{K_t}$ and $S$.


\subsection{Bounds in Bayesian Learning Settings}
The bound in Proposition \ref{Prop:SupremumGP} requires the computation of $\sup_{x \in I_{\hat x_{t}}^{K_t}}|\mu^{i}(x,\hat x_{t})|_1$, $\xi^{(i)}$, $L_{\hat x_t}^{i}$ and $\lambda_1$, which are related to the extrema of the mean and variance of the GP $\mathbf{f}$ in $I_{\hat x_{t}}^{K_t}$ and to a Lipschitz constant on $d^{(i)}_{\hat {\mathrm{x}}_t}$.
In a Bayesian learning setting, these can be computed by relying on the methods discussed in \cite{cardelli2018robustness} and applying them to the GP of \eqref{Eqn:OriginalProcess}.
We here briefly review and adapt to the current setting the methods for the bounding of $\sup_{x \in I_{\hat x_{t}}^{K_t}}|\mu(x,\hat x_{t})|_1$, $\xi^{(i)}$, while we refer to \cite{cardelli2018robustness} for a detailed explanation of how to compute $L_{\hat x_t}^{i}$ and $\lambda_1$ (which are not changed by the control input). 
Let $k^{i}(\cdot,\cdot)$ be the GP kernel function for the $i$-th output dimension, $x \in I_{\hat x_{t}}^{K_t}$ be a test point and $\mathcal{D} = \{(x_j,y_j)\ | j = 1,\ldots, M \}$ a training data set.
Then the mean and variance of the Gaussian process $\mathbf{f}$ conditioned on the training data is given by the set of equations \cite{gpbook}:
\begin{align} 
     \mu_{x}^{\mathbf{f},i} &= k(x , \mathcal{D}) k^{i}(\mathcal{D},\mathcal{D})^{-1} \mathbf{y} \label{eq:mean_inference} \\
     \Sigma_{x,x}^{\mathbf{f},(i,i)} &= k(x,x) - k^{i}(x , \mathcal{D}) k(\mathcal{D},\mathcal{D})^{-1}  k^{i}(x , \mathcal{D})^T  \label{eq:var_inference}
\end{align}
where $\mathbf{y} = [y_1,\ldots,y_M]$.
Assuming continuity and differentiability of the kernel
function $k(\cdot,\cdot)$, it is possible to find linear upper and lower bounds on the
covariance between a test point and a point in the training dataset. 
In the case of the squared exponential kernel it suffices to see that the covariance between a test point $x$ and a training point $x_j$ can be written as a differentiable, convex function of the uni-dimensional auxiliary variable $ z_j = \vert \vert {x} -  x_j \vert \vert$.
As such, by inspection of the derivatives it is possible to find linear coefficients $a_j^L,b_j^L,a_j^U$ and $b_j^U$ such that\footnote{By definition of convex function, a lower bound is given by any tangent to the function (computed through derivative calculations) and an upper bound is given by connecting the extrema of the function in $I_{\hat x_{t}}^{K_t}$.}:
\begin{equation}\label{eq:linear_bounds}
a_j^L + b_j^L \vert \vert x -  x_j \vert \vert \leq k^{i}(x,x_j) \leq a_j^U + b_j^U \vert \vert x -  x_j \vert \vert,\;  \forall x \in I_{\hat x_{t}}^{K_t}.
\end{equation}
These bounds can be propagated through the inference formula for $\mathbf{f}$ by performing the matrix multiplication involved in  (\ref{eq:mean_inference}) and (\ref{eq:var_inference}).
The resulting equations for the mean and variance are respectively linear and quadratic on the auxiliary variable $ z_j = \vert \vert {x} -  x_j \vert \vert$, and can hence be optimised analytically by inspection of the derivatives.
This can then be further refined using a branch and bound optimisation approach over $I_{\hat x_{t}}^K$.

This can be straightforwardly generalised to take into account the extra input dimensions coming from a deterministic control strategy $\pi(x)=u$, without increasing the size of the branch and bound search space, and thus without significantly changing the computational time.
To do so it suffices to solve the optimisation problems $u_j^L = \min_{x \in I_{\hat x_{t}}^{K_t}}  \pi_j(x)$ and $u_j^U = \max_{x \in  I_{\hat x_{t}}^{K_t}} \pi_j(x) $ 
for $j = 1, \ldots, m$, that is computing maximum and minimum of the control allowed in the current state space sub-region. 
Notice that for the policy functions implemented (e.g.\ linear or sum of radial basis functions) this can be computed analytically and in constant time \cite{pilco}. 
The bounds can then be used by treating $u$ and $x$ symmetrically.

\subsubsection{Computational Complexity of Bound}
Computation of the bound involves the calculation of \eqref{eq:linear_bounds} for each point in the training set $\mathcal{D}$, and the computation of the inference formulas \eqref{eq:mean_inference} and \eqref{eq:var_inference} on the resulting bounds. 
This is $\mathcal{O}(M)$ for the mean function and $\mathcal{O}(M^2)$ for the variance (as the latter is quadratic), where $M = \vert \mathcal{D} \vert$  is the number of training samples used.
Refining the bounds with a branch and bound approach has a worst-case cost that is exponential in the dimension of the variable $x$, that is $n$.
Bounding of $L_{\hat x_t}^{i}$ and $\lambda_1$ is done in constant time.
This is iterated for any output dimension of the GP.
After branch and bound has converged, computation of the optimal value for $K_{t+1}$ is linear on the number of candidate values explored, as it involves the computation of the integral in Proposition \ref{Prop:SupremumGP} with known constants.
Finally the procedure is identically repeated for each time step $t$.

\subsection{Using the Safety Guarantees for PILCO}

In this section we briefly examine how the safety guarantees can be used in conjunction with a safe, model-based policy search algorithm,
which extends the Safe PILCO framework \cite{polymenakos2019safe}. 
PILCO's goal is to control an unknown dynamical system throughout a task, by efficiently optimising the parameters $\theta$ of a feedback control policy $\pi^\theta$, implemented originally as a linear controller or a sum of radial basis functions. 
In Safe PILCO, safety considerations are added, with the introduction of constraints, that demand the system to stay in a safe subset of the state space $S\subseteq X$ with high probability. 
Specifically, after a controller is trained using a learned GP model, and before the controller is applied to the controlled system, the probability that this controller violates the constraints is estimated using moment matching. 
Since MM is an approximation that might lead to underestimating the true uncertainty of the iterative predictions (as we show below) controllers that violate the constraints can be allowed to be implemented.
We therefore suggest to replace this step, referred to in (\cite{polymenakos2019safe}) as \textit{a safety check}, with the bounds estimated from Corollary \ref{Corollary:TestTUbeTraectory}. 
This replacement is straightforward and provides better protection from unsafe controllers used in possibly safety critical applications.

\section{Experiments}

In this section we apply the methods presented above to various GPs with SQE kernel trained from data.
In all the experiments we use the bound from Theorem \ref{Prop:BoundErrorRecursive} with the L1 norm, that is with $d=1$. 
First we explicitly compare our formal, guaranteed bounds with the probability estimation obtained by Moment Matching (MM) in two iterative prediction scenarios (with no control involved).
We then investigate in the Mountain Car application \cite{Moore90efficient} the behaviour of our methodology for certification of a given control policy.
Finally we show how to compute bounds for the behaviour of closed-loop systems for a given controller\footnote{Code will be available on \texttt{github}.}. 
GPs are trained with the GPML package, using maximum marginal likelihood for hyperparameter selection.
Candidate policies are either arbitrarily selected for the purpose of demonstration or obtained from PILCO.
They are either linear or linear squashed through a sine wave to constrain the input magnitude \cite{pilco}.
All the experiments were run on a MacBook Pro (Early 2015) with a DDR3 8 GB RAM  @1867 MHz, and an Intel Core i7 processor @3.1 GHz.

\subsection{Iterative Prediction}

We analyse the behaviour of our method in a one-dimensional synthetic dataset where the system dynamics are distributed as a Gaussian at each time step. 
Further we assume that the initial state of the system is Gaussian, that is $\mathbf{x}_0 \sim \mathcal{N} (\mu_0,\Sigma_0)$, with mean and variance given by $\mu_0=0$ and $\Sigma_0 = 0.01$.
We compute predictions and bound the trajectory for an horizon of $H=10$ time steps.
We use $\epsilon = 0.05$, that is we require bounds holding with probability at least $95\%$ and compare with the results obtained by MM.
Namely, we compare our bounds with plus/minus two standard deviation estimated by MM. Notice that when MM is exact (i.e.\ when the system dynamics are effectively Gaussian at each time-step), then this would as well correspond to bounds at $95\%$ probability.
Results for this analysis are given in Figure \ref{fig:known}, where our bound is depicted with a thick red solid line, and MM results are represented by the green shaded area. 
Further, we extract 100 trajectories from the GP, which are depicted with thin colored lines, in order to provide statistical validation for the results. 
Notice that the latter are almost entirely within the MM shaded area. In fact, since the system dynamics are fully Gaussian at each time step, that is $\mathbf{x}_t$ is Gaussian for each $t$, then the approximation made by MM is almost exact and well behaved. 
Note that  our method successfully bounds the sampled trajectories at each time step.

\begin{figure}
\includegraphics[width=\linewidth]{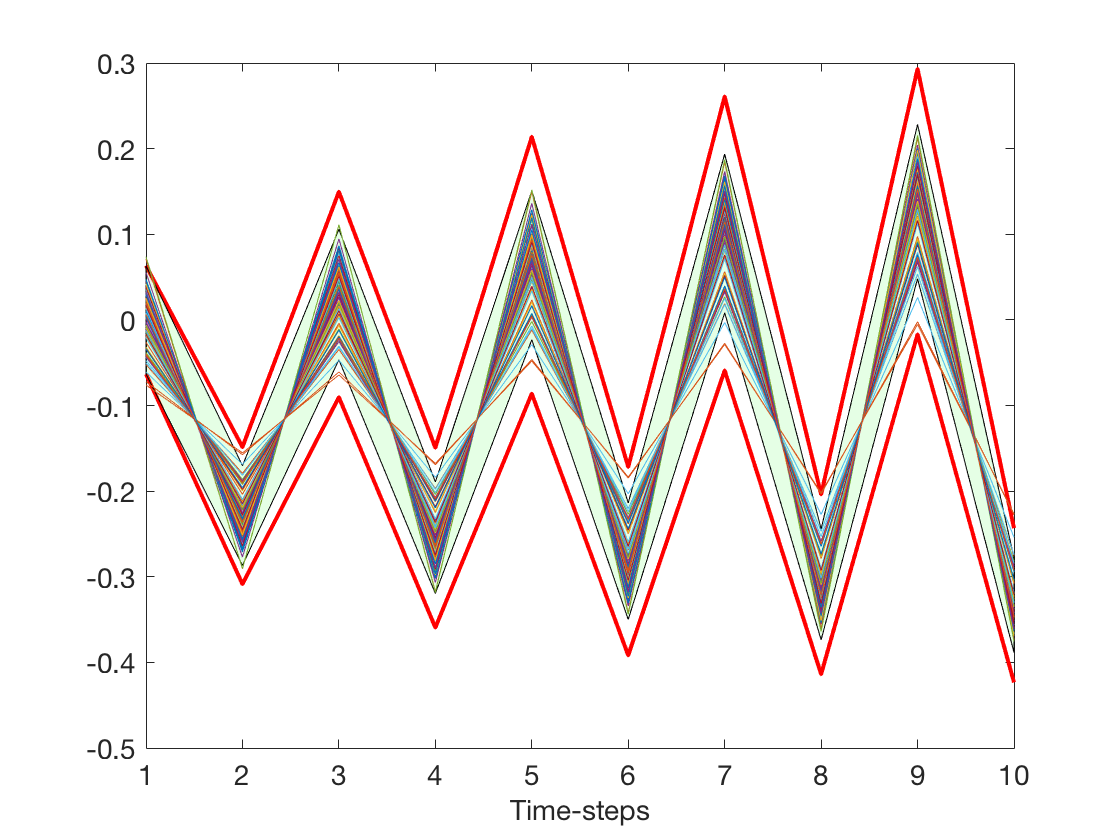}
\caption{
\small
A set of 100 trajectories sampled from a GP (thin coloured line).
The green shaded area corresponds to plus/minus two standard deviations of the moment matching prediction.
The thicker red lines delimit the area with $95\%$ probability according to Theorem \ref{Prop:BoundErrorRecursive}.}
\label{fig:known}
\end{figure}

\begin{figure*}
\centering
\begin{subfigure}[b]{0.3\textwidth}
\includegraphics[width=\linewidth]{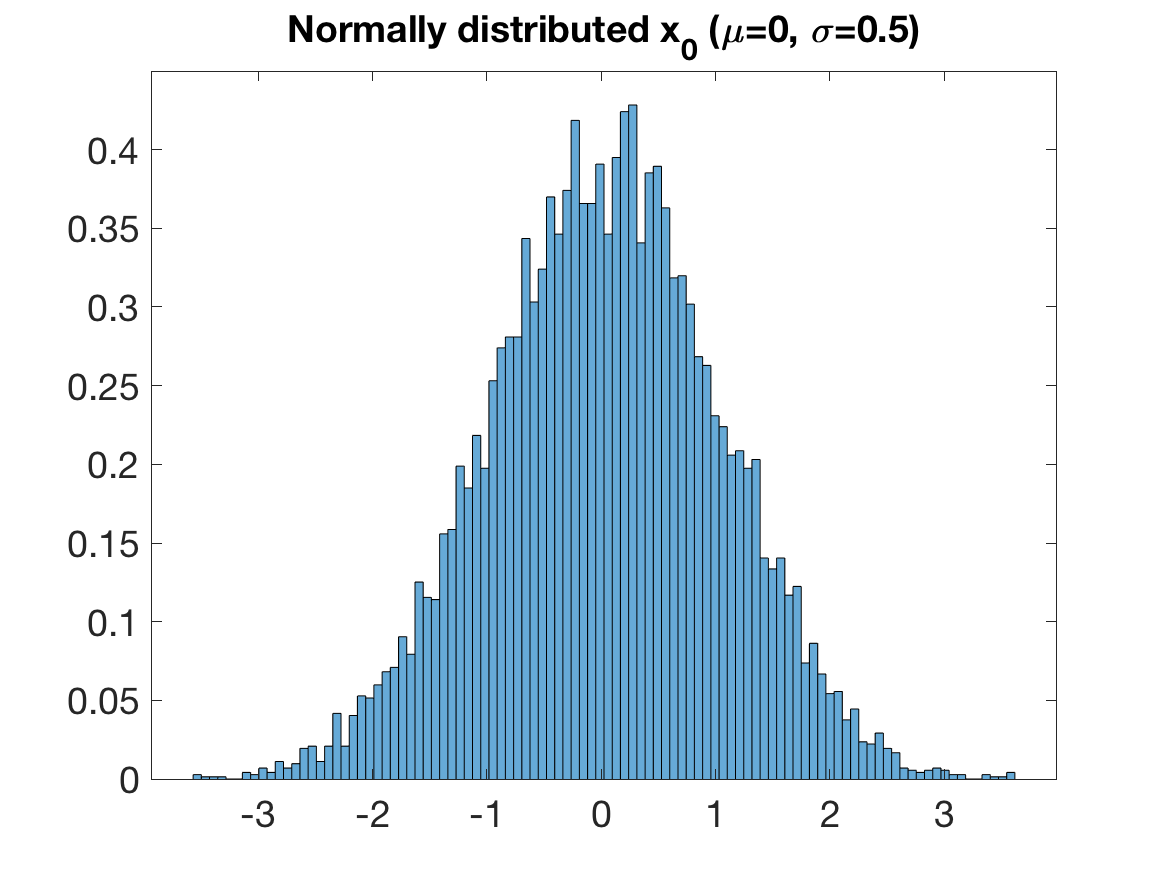}
\label{fig:x0}
\end{subfigure}
\begin{subfigure}[b]{0.3\textwidth}
\includegraphics[width=\linewidth]{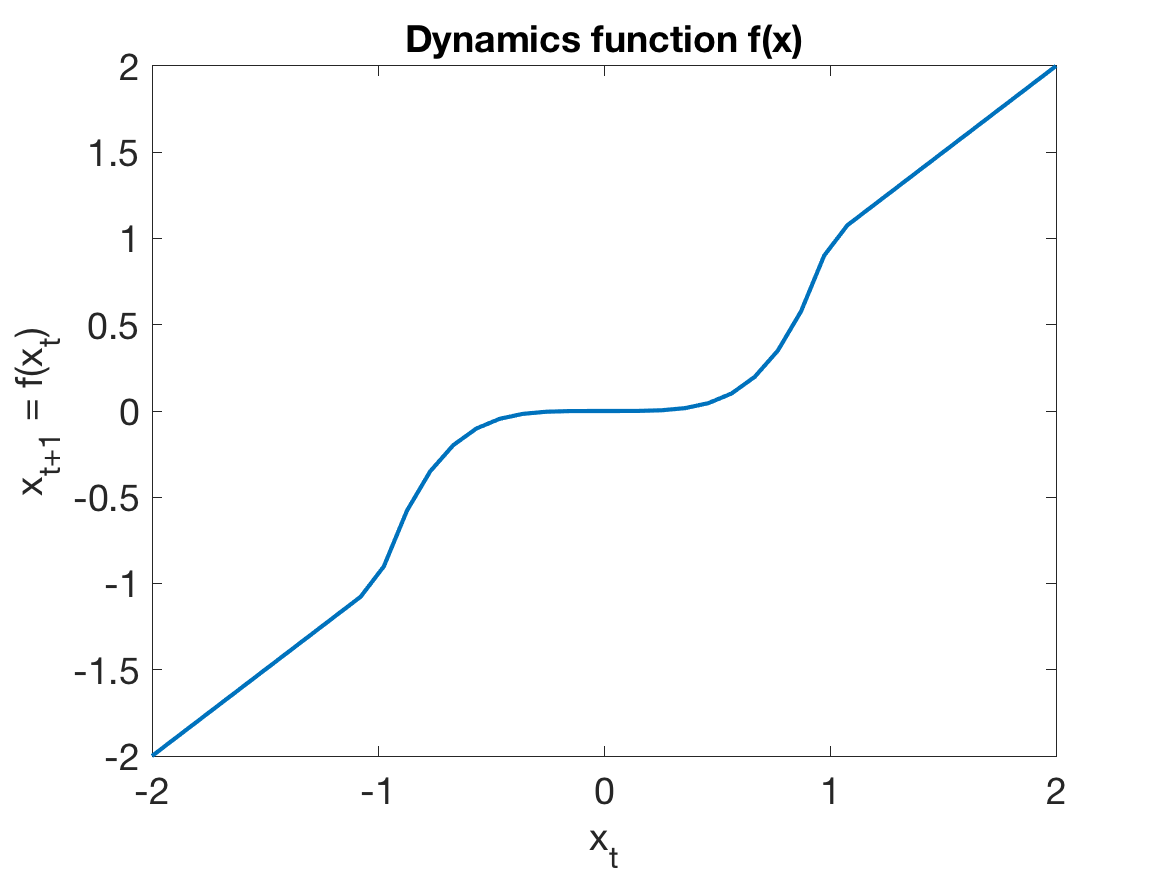}
\label{fig:f}
\end{subfigure}
\begin{subfigure}[b]{0.3\textwidth}
\includegraphics[width=\linewidth]{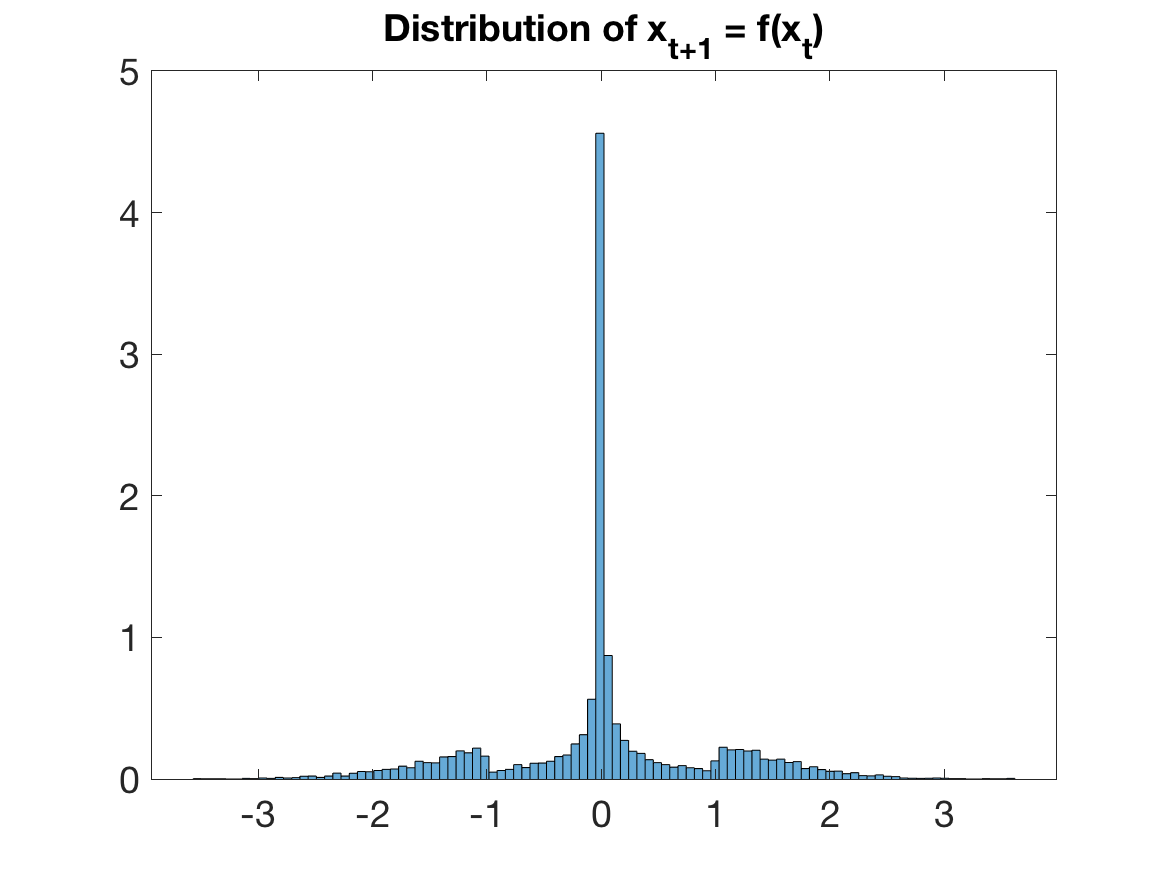}
\label{fig:x1}
\end{subfigure}
\caption{
\small
Initial state distribution, system dynamics and state distribution after a time-step for the system described by the set of Equations \ref{eq:system_dynamics}. 
Histograms show empirical results for 10000 trajectories. On the \textbf{left} is the normally distributed initial state, which passes through the nonlinear dynamics function in the \textbf{middle}, leading to the distribution on the \textbf{right}.}
\label{fig:dyn}
\end{figure*}

\begin{figure*}
\centering
\begin{subfigure}[b]{0.24\textwidth}
\includegraphics[width=\linewidth]{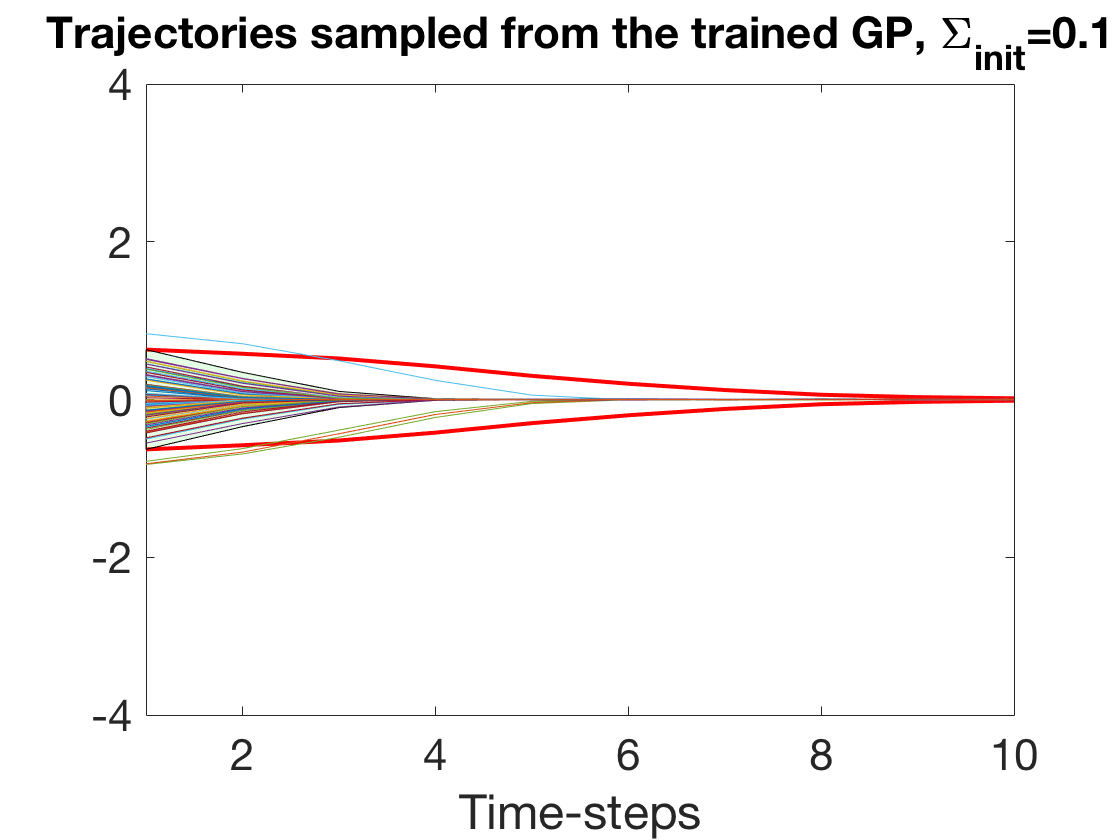}
\end{subfigure}
\begin{subfigure}[b]{0.24\textwidth}
\includegraphics[width=\linewidth]{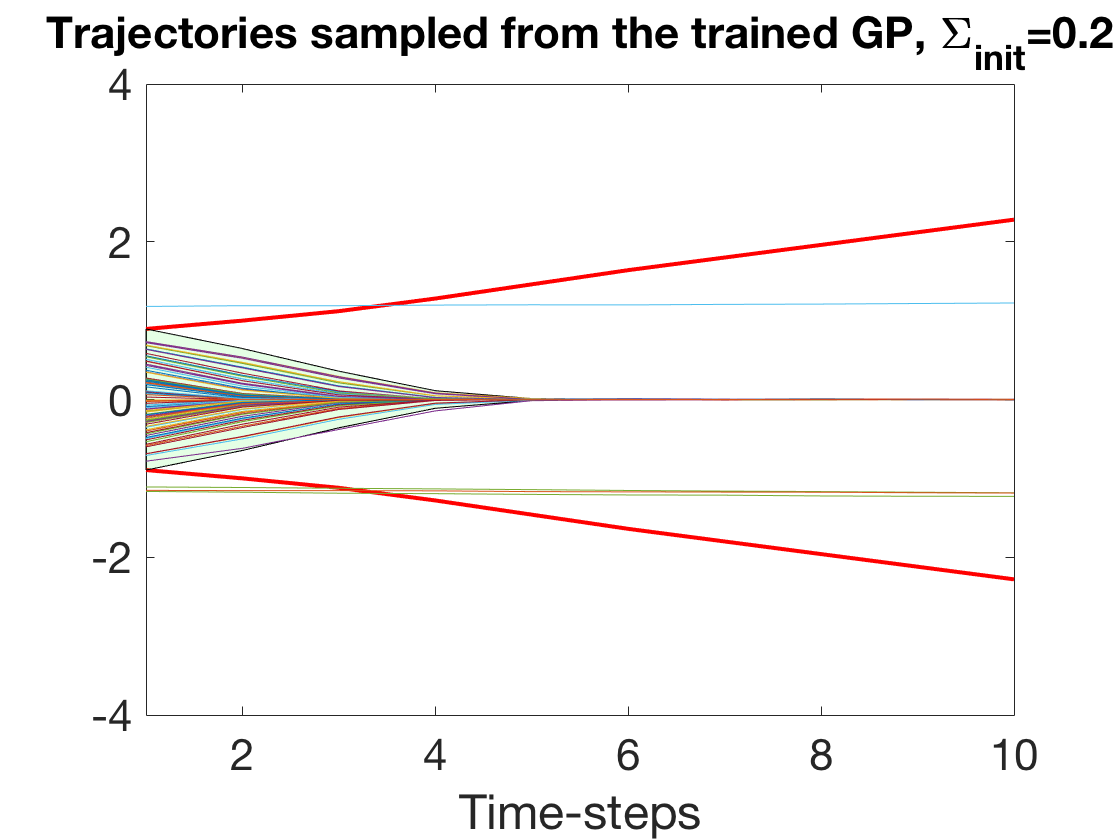}
\end{subfigure}
\begin{subfigure}[b]{0.24\textwidth}
\includegraphics[width=\linewidth]{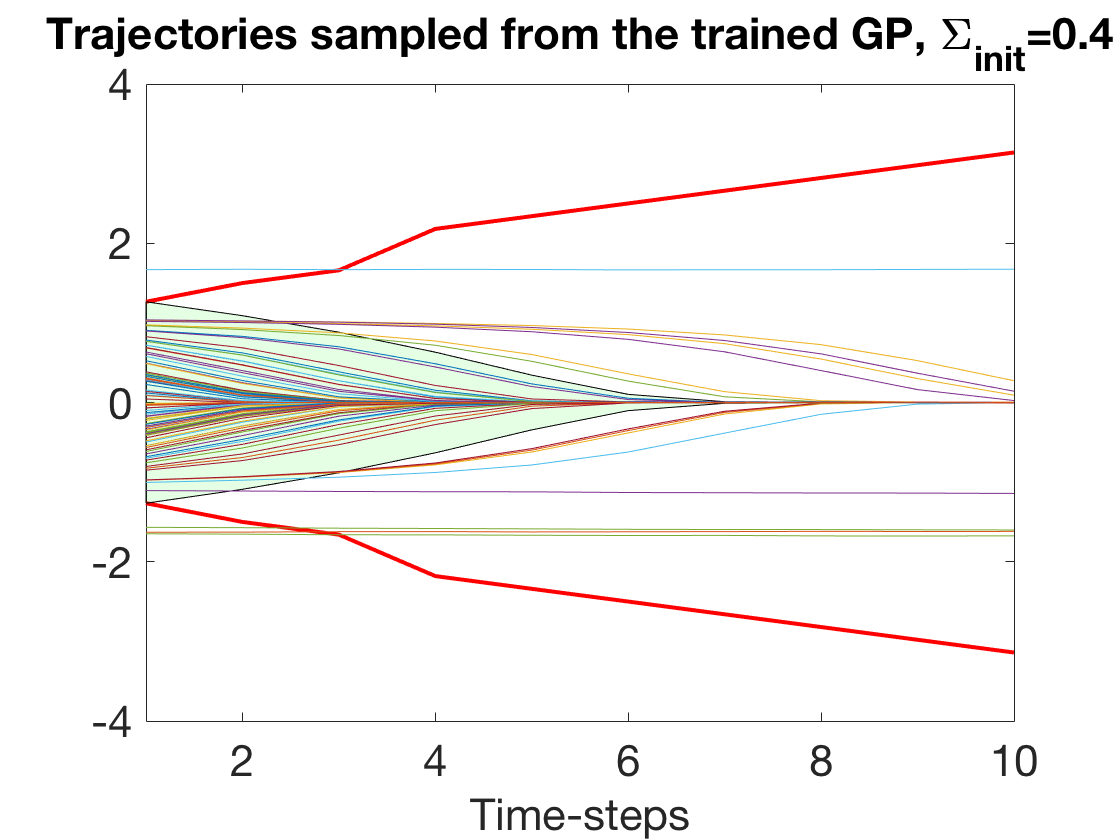}
\end{subfigure}
\begin{subfigure}[b]{0.24\textwidth}
\includegraphics[width=\linewidth]{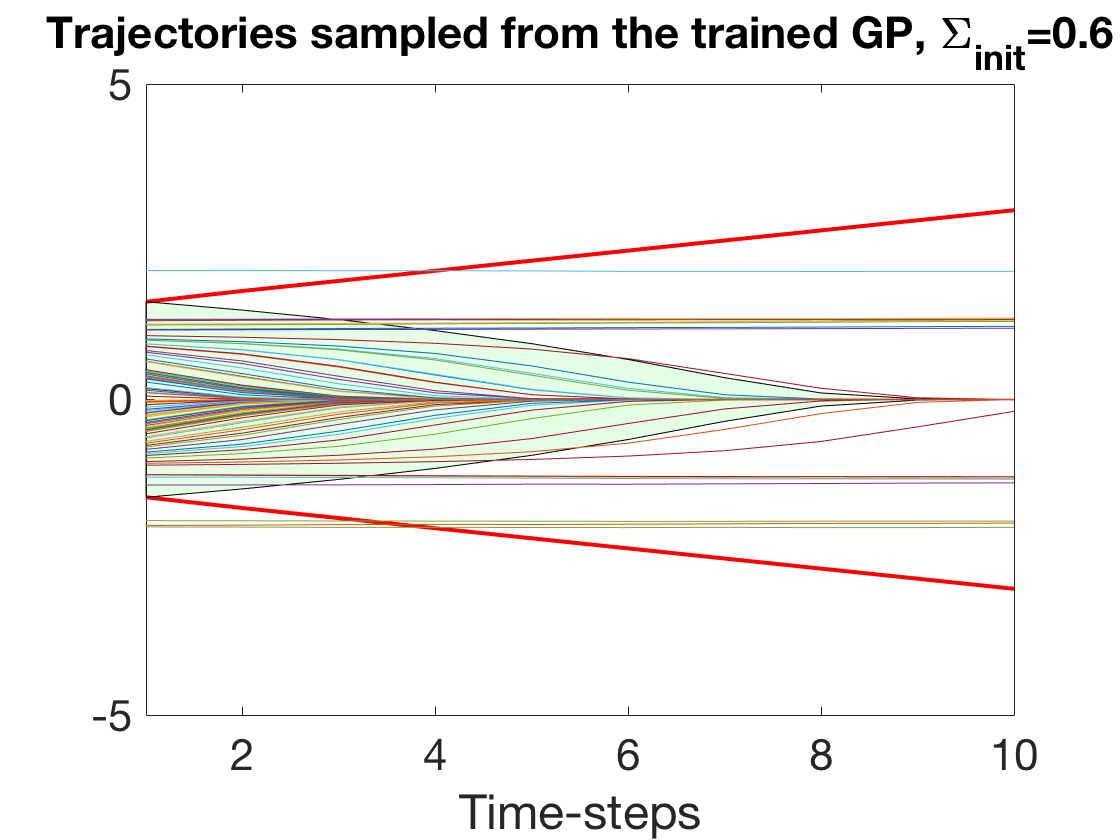}
\end{subfigure}
\caption{
\small
As the initial variance increases, more trajectories, having an initial state $|x_0|>1$,
do not converge to $0$. 
Moment matching fails to account for this fact (green shaded area showing two standards deviations).
Our bound (red line) grows appropriately. Thinner colored lines represent 100 sampled trajectories from the GP.
}
\label{fig:f_traj_GP}
\end{figure*}

Notice that MM succeeds in bounding the GP trajectories because the Gaussian approximation performed by MM is well suited for the example above. 
However, as soon as this does not hold anymore, the results obtained with MM fail to bound the actual GP trajectories.
As an example of this, consider a system with dynamics given by:
\begin{equation}\label{eq:system_dynamics}
    h(x) = 
    \begin{cases}
    \text{sign}(x) x^4, \text{ if } |x| < 1 \\
    x, \text{ otherwise.}
    \end{cases}
\end{equation}
We train a GP on data sampled from this system.
With the function being non-linear, we have that $\mathbf{x}_t$ is non-Gaussian for $t > 0$ (see Figure \ref{fig:dyn}), which implies that application of MM can introduce unaccounted approximation errors.
Furthermore, the specific dynamics chosen are such that the variance predicted with MM will inevitably shrink, leading to a systematic underestimation of the actual region in which GP trajectories are located.
When the initial position of the trajectory, $x_0$, is greater than 1, then the trajectory will constantly be at $x_0$. 
Thus, assuming $x_0  \sim \mathcal{N} (\mu_0,\Sigma_0) $, one can choose a $\Sigma_0$ big enough such that the GP trajectories will be outside any tube parallel to the x-axis with probability greater than $\epsilon$.
However after finitely-many time steps MM variance will wrongly shrink to values very close to zero, hence failing to account for the majority of the probability mass of the GP. 

Empirical results for this system using $\epsilon = 0.05$ are plotted in Figure \ref{fig:f_traj_GP}, for values of initial variance $\Sigma_0$ going from $0.1$ to $0.6$.
The empirical results agree with the discussion above.
If the initial variance is small enough, than the overwhelming majority of GP trajectories will converge to zero.
However, as the initial variance grows, more and more trajectories do not converge. 
MM fails to account for this behavior, and the variance predicted by MM will fail to mirror the actual dynamics of the GP under analysis. 
Note that, our method, being guaranteed to provide correct results, is able to successfully bound (up to $1-\epsilon$ probability) the actual trajectory of the GP, independently of the initial variance. In fact, our method does not rely on any particular assumption, and is able to provide worst-case scenario analysis independently of the general shape of the GP.

\subsection{Open-loop control for Mountain Car}
In this Section we show how our method can be used to certify a control input for a dynamical system. 
The environment we are considering is a version of the continuous mountain car problem \cite{Moore90efficient}.
Briefly, a car is supposed to go up a hill to its right, with a goal state on top of the hill.
Because it does not have enough power to climb the hill directly, it has to go up a hill to the left first to gather momentum.
The state space has two dimensions (position and velocity of the car), and the control input is one dimensional and corresponds to a force applied to the car. 

As previously, we train a GP on data generated from the environment, in this case following a random policy. 
We assume we have access to an initial normal distribution for the starting state 
and we want to evaluate a proposed sequence of actions.
Specifically, we want to perform predictions about the sequence of states (position and velocity) of the car, and to provide high probability bounds for these predictions.
The trained GP model has a 3-dimensional input space, as it takes $(x_t, u_t)$ pairs as inputs, corresponding to the two state-space variables and the control input, 
and 2-dimensional outputs, that correspond to $x_{t+1}$. 
The two output dimensions correspond to two independent GPs, each one predicting a state variable. 
However, the predictions of each model are based on the previous predictions of \emph{both} models.
In more detail, assume a state $x_t \in X \subset \mathbb{R}^2$, where both components of $x_t$ are bounded.
These form a tuple $[x_t^1, x_t^2, u]$, where $x_t^1 \in [lb_1, ub_1]$, and $x_t^2 \in [lb_2, ub_2]$, and the exact value of $u$ is known (as we are verifying an arbitrary, fixed control policy).
This tuple is the input to the two GP models, 
with one of them providing the predicted position $x_{t+1}^1$,
with its new lower and upper bound, and the other one providing the same quantities for the velocity $x_{t+1}^2$.


We train the GP model on a dataset of 500 random actions applied to the mountain car.
Now, for a proposed sequence of actions, we can bound the predicted trajectories, using our method with $\epsilon=0.1$ (that is bound with 90\% probability).
Results from a typical run are presented in Table \ref{table:open_loop}.
Drawing 1000 trajectories from the mountain car system we 
verify that empirically more than 90\% (91.6\%) of them stay within the bounded area around the predictions obtained by our bound. 

\begin{table}
\centering
\resizebox{\linewidth}{!}{
\begin{tabular}{c c c c c c}
\hline \\
t & Control $u$ & $x^1$& $x^2$ & Bound $x^1$ & Bound $x^2$ \\
1 & 1.85  & -0.50 & 0.00 & 0.020 & 0.020    \\ 
2 & -0.97 & -0.38 & 0.53 & 0.030 & 0.080   \\  
3 & 1.39  & -0.37 & -0.49 & 0.055 & 0.125  \\ 
4 & 0.17  & -0.53 & -0.20 & 0.105 & 0.220  \\ 
5 & -1.95 & -0.57 & -0.02 & 0.130 & 0.405  \\ 
6 & -     & -0.87 & -0.05 & 0.225 & 0.595   \\  
\hline \\
\end{tabular}}
\captionsetup{width=0.45\textwidth}
\caption{\small Predictions along with $90\%$ probability bounds for a sequence
of 5 actions applied to the mountain car. Columns $x^1$ and $x^2$ report the mean value of position and velocity of the car. Columns Bound $x^1$ and Bound $x^2$ report the computed the interval around $x^1$ and $x^2$ containing at least $90\%$ of the trajectories.}
\label{table:open_loop}
\vspace{-4mm}
\end{table}


\subsection{Closed-loop control of linear and quadratic systems}
Here we use the proposed method to predict the closed-loop behaviour of several dynamical systems for a proposed feedback controller. 
The systems are either linear, or linear with an added quadratic term, of the general form: 
$$ \dot{x}^i = A^ix + x^T Q^ix + B^i u,$$
where $x^i$ is the $i$-th component of the state vector $x$. We assume a dataset $D = \{x_i, u_i, y_i\}$ of transitions is provided, where $y_t = x_{t+1} = f(x_t, u_t)$ and a candidate controller $C$. 
We train the GP model on 300 data points, and the bounds are calculated with $\epsilon = 0.1$ ($90\%$ probability bounds). 
The controller is either linear, or linear squashed by a sine function, as in PILCO \cite{pilco}. 
The reference point is the origin and the starting region is a hypercube around the origin with size $0.1650$ for each dimension.
In this setting the mean of the predicted states for the system is of secondary importance
(in the linear case it's trivially zero)
and our interest is focused on the width of the bounds on the prediction error. 
Shrinking bounds can be interpreted as similar to a probabilistic notion of stability for the GP model:
shrinking bounds indicate that with the current controller and initial conditions, the model, with high probability, will stay in a (shrinking) region around the origin. 

\begin{table*}[t] 
\centering
\begin{tabular}{c c c c c c c c c c c c}
\hline \\
   & \multicolumn{2}{c}{System 1} & \multicolumn{2}{c}{System 2} & \multicolumn{2}{c}{System 3} & \multicolumn{2}{c}{System 4} & \multicolumn{3}{c}{System 5}  \\
 t$/$Bounds for: & $x$, W=0 & $x$, W=-0.2 & $x^1$ & $x^2$ & $x^1$ & $x^2$ & $x^1$ & $x^2$ & $x^1$ & $x^2$ & $x^3$\\
t=1 & 0.1650 & 0.1650  & 0.1650 & 0.1650 & 0.1650 & 0.1650 & 0.1650 & 0.1650 & 0.1650 & 0.1650 & 0.1650 \\ 
t=2 & 0.1695 & 0.1645  & 0.1610 & 0.1605 & 0.1620 & 0.1610 & 0.1430 & 0.1585 & 0.1650 & 0.1605 & 0.1650 \\  
t=3 & 0.1735 & 0.1640  & 0.1570 & 0.1580 & 0.1595 & 0.1575 & 0.0415 & 0.1525 & 0.1645 & 0.1545 & 0.1650 \\ 
t=4 & 0.1775 & 0.1635  & 0.1525 & 0.1540 & 0.1580 & 0.1545 & 0.0090 & 0.1465 & 0.1620 & 0.1530 & 0.1650 \\ 
t=5 & 0.1815 & 0.1630  & 0.1485 & 0.1505 & 0.1565 & 0.1520 & 0.0050 & 0.1405 & 0.1590 & 0.1515 & 0.1650 \\
t=6 & 0.1855 & 0.1625  & 0.1450 & 0.1475 & 0.1540 & 0.1500 & 0.0050 & 0.1340 & 0.1565 & 0.1470 & 0.1650 \\
Viol. ratio & 0.0732 & 0.0902 & \multicolumn{2}{c}{0.0841} &  \multicolumn{2}{c}{0.0957} & \multicolumn{2}{c}{0.0347} & \multicolumn{3}{c}{0.0659}\\
\hline \\
\end{tabular}
\captionsetup{width=0.90\textwidth}
\caption{
\small
Calculated bounds for different systems over an episode with 5 transitions. As "Viol. ratio", violations ration,  we denote the fraction of transitions for which the bounds (calculated with a tolerance $\epsilon=0.10$) were violated out of the $1000$ sampled trajectories for each system. 
}
\label{table:master}
\end{table*}

For each scenario, once the data and candidate controller is provided we:
\bi
\item 
\item Train a GP model on the provided dataset.
\item Assuming that the model is accurate, use the presented method to make bounded iterative predictions
\item Statistically verify that the bounds are valid by sampling trajectories from the real system (verifying both that the learned model is accurate enough, and that the predicted bounds quantify uncertainty correctly). 
\ei

All results are presented in Table \ref{table:master}.

\subsubsection{System 1, 1-dimensional state space, 1 control input, linear}
In this simple case, we start with a linear, one-dimensional system with one control input.
The parameters take the following values $A = 0.05, Q=0, B=1.0$.
We use a linear controller for this case, so $u = Wx$. 
For the system to be asymptotically stable, we need $A+BW<0 \Leftrightarrow W<-A$.
We estimate the bounds with \emph{no control}, $W=0$, and for a controller that stabilises the system, $W=-0.2$. 
In the first case the bounds \emph{must} be getting wider (since our bounds are conservative), 
while in the second, the bounds should be getting narrower around the origin but that's not guaranteed.
Results show that without a controller the bounds indeed get wider, while with the controller the bounds get narrower.



\subsubsection{System 2, 2-dimensional, 1 control input, linear}
Here we make bounded predictions for a linear system with 2 dimensions and a single control input.
This only incrementally harder than the previous example, since the two dimensions have independent dynamics and
the controller stabilises the first dimension only while the second dimension has inherently convergent
dynamics. The system parameters:
$$A = 
\begin{bmatrix}
0.1 & 0.0 \\
0.0 & -0.4
\end{bmatrix}, B = \begin{bmatrix} 1.0 \\ 0.0 \end{bmatrix}, W = \begin{bmatrix} -0.6 & 0.0 \end{bmatrix}.$$
The bounds on both dimensions contract with time.


\subsubsection{System 3, 2-dimensional, 2 control inputs, linear}
Next we work with a system that's still 2-dimensional with state variables that are not independent, but two control inputs available. The system parameters:
$$A = 
\begin{bmatrix}
0.1 & 0.08 \\
-0.05 & 0.15
\end{bmatrix}, 
B = \begin{bmatrix} 1.0 & 0\\ 0.0 & 1.0 \end{bmatrix}, 
W = \begin{bmatrix} -0.4 & 0.0 \\ 0.0 & -0.5 \end{bmatrix}.$$
As shown in Table \ref{table:master} the bounds contract in this case too.





\subsubsection{System 4, 2-dimensional, 1 control input, quadratic dynamics, controller from PILCO}
Here we train a linear controller squashed by a sine function (effectively bounding the control inputs between $-1$ and $1$) with PILCO \cite{pilco} and then we calculate the bounds for the resulting system.
$$A = 
\begin{bmatrix}
-0.2 & 0.05 \\
-0.05 & -0.4
\end{bmatrix}, 
B = \begin{bmatrix} 1.0 \\ 0.0 \end{bmatrix}, 
Q^1 = \begin{bmatrix} 1.0 & 0.0\\ 0.0 & 0.0\end{bmatrix},$$

$$Q^2 = \begin{bmatrix} 1.0 & 0.0\\ 0.0 & 0.2\end{bmatrix},
W = \begin{bmatrix} -8.61 & -0.02 \end{bmatrix}.$$
The estimated bounds verify convergence.


\subsubsection{System 5, 3-dimensional system, 2 control inputs, linear}
In this example the system is linear and has 3 dimensions and 2 control inputs. Its parameters are:
$$A = 
\begin{bmatrix}
-0.2 & 0.0 & -0.0 \\
0.0  & -0.3 & 0.0 \\
0.0 & 0.0 & -0.6
\end{bmatrix}, 
B = \begin{bmatrix} 1.0 & 0.0 \\ 0.0 & 1.0 \\ 0.0 & 0.0\end{bmatrix}, $$
$$W = \begin{bmatrix} -0.4 & 0.0 & 0.0 \\ 0.0 & -0.2 & 0 \end{bmatrix}.$$
Notice that for the third state variable, even though the system is contractive (by inspecting $A$), 
the bound does not contract (it coincidentally stays constant). 

Overall the results indicate that the bounds can correctly identify contractive behaviour due to the controller.

\section{Conclusions and Future Work}
In this paper, we derived a new formal probabilistic bound for iterated predictions with a GP model, 
without control, in open-loop and in closed-loop scenarios.
Our approach does not make any further assumptions on the properties of the GP, other than knowledge of the kernel hyperparameters, learnt
through maximum marginal likelihood,
and every intermediate quantity used is calculated directly from the data.
The experimental results show that our method is able to correctly propagate uncertainty even when existing heuristic approaches fail.
Furthermore, they showcase how our method can be used to certify the safety of proposed controllers 
on GP models.
In future work, we want to quantify the modelling error (i.e., the error performed in learning the ground truth in the GP training) and its effect on the proposed bounds, 
and further integrate our approach with a model-based reinforcement algorithm like Safe PILCO.

\section{Proofs}
\label{Sec:proofs}
\noindent \emph{Proof of Theorem \ref{Prop:BoundErrorRecursive}}
First we prove the following Lemma:
\begin{lemma}
\label{Lemma:supProbab}
Let $\mathbf{f}(x)$ be a stochastic process. Consider measurable sets $A$ and $B$ Then, it holds that
$$ P(\mathbf{f}(y) \in A | y \in B)\leq  P(\sup_{y \in B} \mathbf{f}(y)\in A  ). $$
\end{lemma}
\begin{proof}{(Sketch)}
To prove Lemma \ref{Lemma:supProbab}  it is enough to note that for each realization of $\mathbf{f}$, $y \in B,$ and measurable $g$ we have that $g(\mathbf{f}(y))\leq \sup_{y^* \in B} g(\mathbf{f}(y^*))$. Hence, we can conclude by taking the expectation.  
\end{proof}

Now the following calculations follow
\begin{align*}
    &P(\mathbf{e}_{t+1} > K_{t+1})\\
    &\text{(By Definition of $\mathbf{e}_t$)} \\
    =&P(|g(\hat x_t,u_t) -\mathbf{f}(\mathbf{x}_t,u_t)|_1 > K_{t+1})\\
	&\, \text{(By Marginalising with the events $\mathbf{e}_t>K_t$, $\mathbf{e}_t \leq  K_t$)} \\
	\leq & P(|g(\hat x_t,u_t)- \mathbf{f}(\mathbf{x}_t, u_t)|_1>{K_{t+1}}{}\, | \, \mathbf{e}_t\leq K_t)P(\mathbf{e}_t\leq K_t) \\
	&\quad + P(\mathbf{e}_t>K_t)   \\
	&\,\text{(By Lemma \ref{Lemma:supProbab} )} \\
	\leq & P(\sup_{x \in I_{\hat x_{t}}^{K_t}}|g(\hat x_t,u_t) - \mathbf{f}(x, u_t)|_1>{K_{t+1}}{})P(\mathbf{e}_t\leq K_t)\\ 
	&\quad +P(\mathbf{e}_t > K_t) \\
	&\, \text{(By the fact that $P(\mathbf{e}_t\leq K_t)=1-P(\mathbf{e}_t > K_t)$)} \\
	=& P(\sup_{x \in I_{\hat x_{t}}^{K_t}}|g(\hat x_t,u_t) - \mathbf{f}(x, u_t)|_1>{K_{t+1}}{})(1-P(\mathbf{e}_t > K_t)) \\
	&+ P(\mathbf{e}_t > K_t).
\end{align*}

 \section*{Acknowledgments}
 This work has been partially supported by the EU's Horizon 2020 program under
 the Marie Sk\l{}odowska-Curie grant No 722022, EPSRC AIMS CDT grant
 EP/L015987/1, the ERC under the European Union’s Horizon 2020 research and innovation programme (grant agreement No.~834115), the EPSRC Programme Grant on Mobile Autonomy (EP/M019918/1) and Schlumberger.
\FloatBarrier
 



\bibliographystyle{ieeetr}  
\bibliography{cite}  

\begin{thebibliography}{10}

\bibitem{gpbook}
C.~E. Rasmussen and C.~K.~I. Williams, ``Gaussian processes for machine
  learning,'' 2006.

\bibitem{girard2003gaussian}
A.~Girard, C.~E. Rasmussen, J.~Q. Candela, and R.~Murray-Smith, ``{G}aussian
  process priors with uncertain inputs application to multiple-step ahead time
  series forecasting,'' in {\em Advances in neural information processing
  systems}, pp.~545--552, 2003.

\bibitem{Vinogradska2018}
J.~Vinogradska, B.~Bischoff, J.~Achterhold, T.~Koller, and J.~Peters,
  ``Numerical quadrature for probabilistic policy search,'' {\em IEEE
  Transactions on Pattern Analysis \& Machine Intelligence}, 2018.

\bibitem{ale_survey}
A.~Abate, ``Formal verification of complex systems: model-based and data-driven
  methods,'' in {\em Proceedings of the 15th ACM-IEEE International Conference
  on Formal Methods and Models for System Design}.

\bibitem{green2012linear}
M.~Green and D.~J. Limebeer, {\em Linear robust control}.
\newblock Courier Corporation, 2012.

\bibitem{kahn17}
G.~Kahn, A.~Villaflor, V.~Pong, P.~Abbeel, and S.~Levine, ``Uncertainty-aware
  reinforcement learning for collision avoidance,'' vol.~abs/1702.01182, 2017.

\bibitem{vuong2019uncertainty}
T.-L. Vuong and K.~Tran, ``Uncertainty-aware model-based policy optimization,''
  {\em arXiv preprint arXiv:1906.10717}, 2019.

\bibitem{candela2003propagation}
J.~Q. Candela, A.~Girard, J.~Larsen, and C.~E. Rasmussen, ``Propagation of
  uncertainty in {B}ayesian kernel models-application to multiple-step ahead
  forecasting,'' in {\em IEEE International Conference on Acoustics, Speech,
  and Signal Processing (ICASSP'03)}.

\bibitem{pilco}
M.~P. Deisenroth and C.~E. Rasmussen, ``{PILCO}: {A} model-based and
  data-efficient approach to policy search,'' in {\em In Proceedings of the
  International Conference on Machine Learning}, 2011.

\bibitem{pil_thesis}
M.~P. Deisenroth, {\em Efficient reinforcement learning using {Gaussian}
  processes}.
\newblock PhD thesis, Karlsruhe Institute of Technology, 2010.

\bibitem{pilco2}
M.~P. Deisenroth, C.~E. Rasmussen, and D.~Fox, ``Learning to control a low-cost
  manipulator using data-efficient reinforcement learning,'' in {\em Robotics:
  Science and Systems}, 2011.

\bibitem{deisenroth2014multi}
M.~P. Deisenroth, P.~Englert, J.~Peters, and D.~Fox, ``Multi-task policy search
  for robotics,'' in {\em 2014 IEEE International Conference on Robotics and
  Automation (ICRA)}, pp.~3876--3881, IEEE, 2014.

\bibitem{kupcsik2013data}
A.~G. Kupcsik, M.~P. Deisenroth, J.~Peters, and G.~Neumann, ``Data-efficient
  generalization of robot skills with contextual policy search,'' in {\em
  Twenty-Seventh AAAI Conference on Artificial Intelligence}, 2013.

\bibitem{mcallister2017data}
R.~McAllister and C.~E. Rasmussen, ``Data-efficient reinforcement learning in
  continuous state-action gaussian-pomdps,'' in {\em Advances in Neural
  Information Processing Systems 30}.

\bibitem{pilconn}
Y.~Gal, R.~T. McAllister, and C.~E. Rasmussen, ``Improving {PILCO} with
  {Bayesian} neural network dynamics models,'' in {\em Data-Efficient Machine
  Learning workshop}, vol.~951, p.~2016, 2016.

\bibitem{polymenakos2019safe}
K.~Polymenakos, A.~Abate, and S.~Roberts, ``Safe policy search using {G}aussian
  process models,'' in {\em Proceedings of the 18th International Conference on
  Autonomous Agents and Multi Agent Systems}, pp.~1565--1573, IFAAMS, 2019.

\bibitem{safe_survey}
J.~Garc{{\'i}}a and F.~Fern{{\'a}}ndez, ``A comprehensive survey on safe
  reinforcement learning,'' {\em Journal of Machine Learning Research},
  vol.~16, pp.~1437--1480, 2015.

\bibitem{vinogradska2016stability}
J.~Vinogradska, B.~Bischoff, D.~Nguyen-Tuong, A.~Romer, H.~Schmidt, and
  J.~Peters, ``Stability of controllers for gaussian process forward models,''
  in {\em International Conference on Machine Learning}, pp.~545--554, 2016.

\bibitem{Koller2018}
T.~Koller, F.~Berkenkamp, M.~Turchetta, and A.~Krause, ``Learning-based model
  predictive control for safe exploration and reinforcement learning,'' {\em
  CoRR}, vol.~abs/1803.08287, 2018.

\bibitem{srinivas2012information}
N.~Srinivas, A.~Krause, S.~M. Kakade, and M.~W. Seeger, ``Information-theoretic
  regret bounds for gaussian process optimization in the bandit setting,'' {\em
  IEEE Transactions on Information Theory}, vol.~58, no.~5, pp.~3250--3265,
  2012.

\bibitem{cardelli2018robustness}
L.~Cardelli, M.~Kwiatkowska, L.~Laurenti, and A.~Patane, ``Robustness
  guarantees for {B}ayesian inference with {G}aussian processes,'' in {\em
  Proceedings of the AAAI Conference on Artificial Intelligence}, vol.~33,
  pp.~7759--7768, 2019.

\bibitem{blaas2019robustness}
A.~Blaas, A.~Patane, L.~Laurenti, L.~Cardelli, M.~Kwiatkowska, and S.~Roberts,
  ``Adversarial robustness guarantees for classification with gaussian
  processes,'' {\em International Conference on Artificial Intelligence and
  Statistics}, pp.~3372--3382, 2020.

\bibitem{adler2009random}
R.~J. Adler and J.~E. Taylor, {\em Random fields and geometry}.
\newblock Springer Science \& Business Media, 2009.

\bibitem{Moore90efficient}
A.~W. Moore, {\em Efficient Memory-based Learning for Robot Control}.
\newblock PhD thesis, University of Cambridge, 1990.

\end{thebibliography}
\end{document}